\documentclass[10pt, conference]{IEEEtran}
\IEEEoverridecommandlockouts
\usepackage{cite}
\usepackage{amsmath,amssymb,amsfonts}
\usepackage{amsthm}
\usepackage{graphicx}
\usepackage[caption=false,font=footnotesize]{subfig}
\usepackage{textcomp}
\usepackage{xcolor}

\newtheorem{proposition}{Proposition}

\usepackage{physics}
\usepackage{float}
\usepackage{mathtools}
\usepackage{bm}
\usepackage{booktabs}
\usepackage{multirow}

\usepackage{hyperref}
\usepackage{algorithm}
\usepackage{algpseudocode}

\newcommand{\ii}{\mathrm{i}}
\newcommand{\e}{\mathrm{e}}

\def\BibTeX{{\rm B\kern-.05em{\sc i\kern-.025em b}\kern-.08em
    T\kern-.1667em\lower.7ex\hbox{E}\kern-.125emX}}
\begin{document}

\title{Self-Regulating Annealing in Heavy-Tailed Diffusion Models}

\author{\IEEEauthorblockN{Keito Wakatsuki}
\IEEEauthorblockA{
\textit{Graduate School of Informatics} \\ \textit{Kyoto University, Kyoto, Japan}
\\wakatsuki.keito.63c@st.kyoto-u.ac.jp}
\and
\IEEEauthorblockN{Hideaki Shimazaki}
\IEEEauthorblockA{
\textit{Graduate School of Informatics} \\ \textit{Kyoto University, Kyoto, Japan}
\\ h.shimazaki@i.kyoto-u.ac.jp }}

\maketitle

\begin{abstract}
  Diffusion models have emerged as a leading framework for deep generative modeling.
  While the standard Gaussian formulation is theoretically convenient, its suitability for heavy-tailed datasets remains unclear.
  To address this, heavy-tailed diffusion models (HTDMs) \cite{pandey2024heavytailed} extend the standard formulation by replacing the Gaussian distribution with a Student's $t$-distribution, thereby improving tail fidelity on heavy-tailed datasets.
  Although stochastic differential equation (SDE)-based sampling is possible in HTDMs, it has not been fully explored.
  In this paper, we propose an SDE-based sampler for HTDMs that explicitly incorporates a state-dependent diffusion coefficient.
  This state dependence naturally induces a self-regulating annealing mechanism by adaptively modulating the effective noise scale.
  We theoretically explore this mechanism and experimentally verify its necessity for reproducing samples from a heavy-tailed distribution.
\end{abstract}

  
\section{Introduction}
Diffusion models have emerged as a leading framework for deep generative modeling in recent years, particularly in image generation and other computer vision tasks \cite{sohl2015deep, yang2023diffusion}.
Denoising diffusion probabilistic models (DDPMs) \cite{ho2020denoising} and their continuous-time formulations via stochastic differential equations (SDEs) and ordinary differential equations (ODEs) \cite{song2020score} provide a foundational framework and continue to attract substantial research interest.

In standard diffusion models, both the forward diffusion process and the reverse generative process are formulated using Gaussian transitions.
While this Gaussian formulation yields analytically tractable marginals and posteriors, as well as simple training objectives, its suitability for heavy-tailed or outlier-prone datasets remains unclear.
Motivated by this concern, recent studies have explored \emph{non-Gaussian} diffusion models that replace the Gaussian formulation with heavy-tailed alternatives, such as approaches based on $\alpha$-stable Lévy processes \cite{yoon2023score} and the Student's $t$-distribution \cite{pandey2024heavytailed}.
In particular, Pandey et al.\ \cite{pandey2024heavytailed} propose heavy-tailed diffusion models (HTDMs), a Student's $t$-based extension that is readily compatible with standard diffusion models and requires only minimal modifications.
Compared with standard diffusion models, HTDMs offer controllable tail behavior and better coverage of outlier regions.

In parallel, there has been growing interest in theoretically analyzing the generative dynamics of diffusion models from a statistical-physics perspective \cite{ambrogioni2023statistical, raya2024spontaneous}. 
In particular, these dynamics have been shown to be theoretically connected to the memory retrieval dynamics of associative memory models \cite{ambrogioni2024search, hoover2023memory}.
Associative memory models have also been extended using deformed exponential families, including the Student's $t$-distribution \cite{aguilera2025explosive}. 
This extension yields curved neural networks that exhibit rich and complex phenomena, such as explosive phase transitions with hysteresis and increased memory capacity. 
These phenomena arise from the stochastic dynamics of memory retrieval driven by a \emph{self-regulating annealing mechanism}, which modulates memory retrieval via state-dependent changes in the effective temperature induced by higher-order statistical structure inherent to the model.

While similar dynamical properties may also arise in diffusion models extended with deformed exponential families, the effects of such deformations have not been fully explored.
For HTDMs, Pandey et al.~\cite{pandey2024heavytailed} derive both an ODE and an SDE from the Student's $t$-based reverse kernel.
However, the empirical evaluation in \cite{pandey2024heavytailed} primarily focuses on the ODE, leaving the SDE largely unexplored.
To bridge this gap, we conduct a detailed analysis of the SDE for HTDMs and propose a novel SDE-based sampler. 
Our proposed sampler incorporates a state-dependent diffusion coefficient, which we argue is necessary to reproduce samples from a heavy-tailed distribution. 
Furthermore, we theoretically demonstrate that this state dependence induces the self-regulating annealing mechanism. 
The effective annealing temperature is automatically modulated based on the distance between the current and target states at each time step.
This mechanism is closely related to the dynamics of memory retrieval in deformed associative memory models.

\vspace{0.25em}
\noindent Our main contributions are summarized as follows:
\begin{itemize}
  \item We propose an SDE-based sampler for HTDMs that incorporates a state-dependent diffusion coefficient.
  \item We demonstrate that this state dependence induces a self-regulating annealing mechanism.
  \item We empirically show that our SDE-based sampler improves tail fidelity on a synthetic dataset relative to a Gaussian baseline and an ablated SDE variant without the state-dependent diffusion coefficient.
\end{itemize}
\begin{flushleft}
  Code availability: \url{https://github.com/keito93/self-regulating-annealing-in-htdms.git}
\end{flushleft}

\section{Background}
\subsection{Diffusion Models}
\label{subsec:diffusion}

We briefly review the standard formulation of diffusion models as latent-variable models \cite{ho2020denoising, luo2022understanding}. 
Let $\mathbf{x}_0 \in \mathbb{R}^d$ denote an observed variable drawn from the data distribution $p_{\mathrm{data}}$.
The latent variables $\mathbf{x}_{\Delta t}, \mathbf{x}_{2\Delta t}, \ldots, \mathbf{x}_{N \Delta t} \in \mathbb{R}^d$ are generated by a fixed diffusion process $q$ that gradually corrupts $\mathbf{x}_0$ into pure noise.
The generative model $p_{\bm{\theta}}$, parameterized by a neural network, defines a generative process that aims to reverse the diffusion process.
When both $q$ and $p_{\bm{\theta}}$ are modeled as Markov chains (as in standard diffusion models), the negative evidence lower bound (negative ELBO) is given by
\begin{small}
\begin{align}
  \mathcal{L}(\bm{\theta}) =& -\mathbb{E}_{q(\mathbf{x}_{\Delta t} \mid \mathbf{x}_0)}[\ln p_{\bm{\theta}}(\mathbf{x}_0 \mid \mathbf{x}_{\Delta t})]
  + \mathcal{D}_0(q(\mathbf{x}_T \mid \mathbf{x}_0) \,\Vert\, p(\mathbf{x}_T)) \notag\\
  &+ \sum_{t = 2\Delta t}^{T} \mathbb{E}_{q(\mathbf{x}_t \mid \mathbf{x}_0)}\big[\mathcal{D}_{0}(q(\mathbf{x}_{t - \Delta t} \mid \mathbf{x}_t, \mathbf{x}_0) \,\Vert\, p_{\bm{\theta}}(\mathbf{x}_{t - \Delta t} \mid \mathbf{x}_t))\big],
  \label{objective function 1}
\end{align}
\end{small}where $T \coloneq N \Delta t$, $t \in \{\Delta t, 2\Delta t, \ldots, T\}$ denotes the discrete time step, and $\mathcal{D}_0$ denotes the Kullback--Leibler (KL) divergence.
The second term in \eqref{objective function 1} matches the prior $p(\mathbf{x}_T)$ of the generative process to the marginal $q(\mathbf{x}_T \mid \mathbf{x}_0)$ of the diffusion process.
The third term, which we refer to as the \emph{denoising matching term}, requires the trainable reverse kernel $p_{\bm{\theta}}(\mathbf{x}_{t -\Delta t} \mid \mathbf{x}_t)$ to approximate the true posterior $q(\mathbf{x}_{t -\Delta t} \mid \mathbf{x}_t, \mathbf{x}_0)$.
In practice, this term serves as the training objective.
To make the denoising matching term tractable, it is common to choose the forward process $q(\mathbf{x}_t \mid \mathbf{x}_{t - \Delta t})$ 
so that both the marginal $q(\mathbf{x}_t \mid \mathbf{x}_0)$ and the posterior $q(\mathbf{x}_{t-\Delta t} \mid \mathbf{x}_t,\mathbf{x}_0)$ can be obtained in closed form.
If $q(\mathbf{x}_t \mid \mathbf{x}_{t-\Delta t})$ is a linear Gaussian transition, then the corresponding marginal and posterior are also Gaussian and analytically tractable.
In this case, with an appropriate choice of the reverse kernel $p_{\bm{\theta}}(\mathbf{x}_{t-\Delta t}\mid \mathbf{x}_t)$, the denoising matching term reduces to a mean-squared-error (MSE) loss.

After training, sampling starts from $\mathbf{x}_T \sim p(\mathbf{x}_T)$, and a synthetic sample $\mathbf{x}_0$ is obtained by iteratively applying the reverse kernel $p_{\bm{\theta}}(\mathbf{x}_{t-\Delta t} \mid \mathbf{x}_t)$.
In the continuous-time limit $\Delta t \to 0$, this discrete-time generative process converges to an SDE, which can be solved numerically to generate samples.
While Gaussian diffusion models are highly successful in practice, it remains unclear whether the Gaussian formulation is well suited to heavy-tailed or outlier-prone datasets, which motivates the heavy-tailed extension discussed next.

\subsection{Heavy-Tailed Diffusion Models}
\label{subsec:htdm}

Pandey et al.\ \cite{pandey2024heavytailed} propose HTDMs and instantiate them as $t$-EDM, a Student's $t$-based extension of the EDM design of Karras et al.\ \cite{karras2022elucidating}.
We follow their construction and slightly reformulate it to align with our subsequent discussion.

They start from the same negative ELBO as in \eqref{objective function 1}.
The forward diffusion process is defined as an $Nd$-dimensional multivariate Student's $t$-distribution over all time steps:
\begin{align}
  &q(\mathbf{x}_{\Delta t}, \mathbf{x}_{2\Delta t}, \ldots, \mathbf{x}_{N \Delta t} \mid \mathbf{x}_0)
  \coloneq \mathrm{St}_{Nd}(\bm{\mu}, \bm{\Sigma}, \nu), \label{HTDMs diffusion process1}\\
  &\quad \bm{\mu} \coloneq [\mu_{\Delta t}, \mu_{2\Delta t}, \ldots, \mu_{N \Delta t}]^\top \otimes \mathbf{x}_0,\\
  &\quad \bm{\Sigma} \coloneq \widetilde{\bm{\Sigma}} \otimes \mathbf{I}_d, \label{tridiagonal matrix 1}
\end{align}
where $\nu > 2$ denotes the degrees of freedom
and $\widetilde{\bm{\Sigma}}$ is an $N\times N$ symmetric tridiagonal matrix whose diagonal entries are
$(\widetilde{\bm{\Sigma}})_{k,k}=\sigma_{k\Delta t}^2$ for $k=1, 2, \ldots,N$,
and whose first off-diagonal entries are
$(\widetilde{\bm{\Sigma}})_{k,k+1}=(\widetilde{\bm{\Sigma}})_{k+1,k}=c_{(k+1)\Delta t}^2$ for $k=1, 2, \ldots,N-1$.
$\mathbf{I}_d$ denotes the $d \times d$ identity matrix, and $\otimes$ denotes the Kronecker product.
Note that under this definition, the diffusion process is non-Markovian.

A key property of this formulation is that the marginals and posteriors in the denoising matching term can be obtained in closed form by marginalization \cite{ding2016conditional}:
\begin{align}
  q(\mathbf{x}_{t} \mid \mathbf{x}_{0}) &= \mathrm{St}_d(\mu_t \mathbf{x}_0, \sigma_t^2 \mathbf{I}_d, \nu), \label{HTDMs diffusion process2}\\
  q(\mathbf{x}_{t - \Delta t} \mid \mathbf{x}_t, \mathbf{x}_0) &= \mathrm{St}_d \left(\bar{\bm{\mu}}_t, \frac{\nu + \Delta_t^2}{\nu + d}\bar{\sigma}_{t}^2 \mathbf{I}_d, \nu + d \right), \label{HTDMs posterior1} \\
  \bar{\bm{\mu}}_t &\coloneq \frac{c_t^2}{\sigma_t^2} \mathbf{x}_t + \left(\mu_{t - \Delta t} - \frac{c_t^2}{\sigma_t^2} \mu_t\right) \mathbf{x}_0, \label{barmu 1}\\
  \bar{\sigma}_t^2 &\coloneq \sigma_{t - \Delta t}^2 - \frac{c_t^4}{\sigma_t^2}, \label{barsigma 1}\\
  \Delta_t^2 &\coloneq \frac{1}{\sigma_t^2} \lVert \mathbf{x}_t - \mu_t \mathbf{x}_0 \rVert^2, \label{state-dependent term1}
\end{align}
where $\Delta_t^2$ denotes the squared Mahalanobis distance.
The schedule of $\mu_t$ is designed so that the marginal $q(\mathbf{x}_T \mid \mathbf{x}_0)$ approximately matches the prior $p(\mathbf{x}_T)=\mathrm{St}_d(\bm{0},\sigma_T^2\mathbf{I}_d,\nu)$, 
and $\sigma_t$ is chosen to increase monotonically with $t$.

Based on the posterior in \eqref{HTDMs posterior1}, Pandey et al.\ \cite{pandey2024heavytailed} define the trainable reverse kernel $p_{\bm{\theta}}(\mathbf{x}_{t - \Delta t} \mid \mathbf{x}_t)$ as a Student's $t$-distribution.
Since the KL divergence between Student's $t$-distributions is intractable, they replace the KL divergence in the denoising matching term with a scaled Eguchi $\gamma$-power divergence \cite{Eguchi2021PythagorasIG, kim2023t}.
This scaled $\gamma$-power divergence is an extension of the KL divergence, to which it converges in the limit $\gamma \to 0$.
Furthermore, to simplify the loss function, they omit the state-dependent factor $(\nu + \Delta_t^2)/(\nu + d)$ in the posterior covariance and approximate the posterior as
\begin{align}
  q(\mathbf{x}_{t - \Delta t} \mid \mathbf{x}_t, \mathbf{x}_0) \approx \mathrm{St}_d \left(\bar{\bm{\mu}}_t, \bar{\sigma}_{t}^2 \mathbf{I}_d, \nu + d \right) \label{HTDMs posterior2}.
\end{align}
Using this simplified posterior, they parameterize the reverse kernel as
\begin{align}
  &p_{\bm{\theta}}(\mathbf{x}_{t- \Delta t} \mid \mathbf{x}_{t}) = \mathrm{St}_d \left(\bm{\mu}_{\bm{\theta}}(\mathbf{x}_t, t), \bar{\sigma}_t^2 \mathbf{I}_d, \nu + d\right), \label{HTDMs model1} \\
  &\quad \bm{\mu}_{\bm{\theta}}(\mathbf{x}_t, t) \coloneq \frac{c_{t}^2}{\sigma_t^2} \mathbf{x}_t + \left(\mu_{t-\Delta t} - \frac{c_{t}^2}{\sigma_t^2} \mu_t\right) \bm{D}_{\bm{\theta}}(\mathbf{x}_t, \sigma_t).
\end{align}
Here, $\bm{D}_{\bm{\theta}}(\mathbf{x}_t, \sigma_t) : \mathbb{R}^d \times \mathbb{R}_{+} \to \mathbb{R}^d$ is a neural network that predicts $\mathbf{x}_0$ from $\mathbf{x}_t$, and aims to approximate $\mathbb{E}[\mathbf{x}_0 \mid \mathbf{x}_t]$.
$\bm{D}_{\bm{\theta}}(\mathbf{x}_t, \sigma_t)$ thus serves as a \emph{denoiser}.
With the scaled $\gamma$-power divergence and the above parameterization, the denoising matching term reduces to the familiar MSE loss:
\begin{align}
  \mathbb{E}_{t \sim \mathcal{U}(\{\Delta t, 2 \Delta t, \ldots,T\}), \mathbf{x}_0 \sim p_{\text{data}}, \mathbf{x}_t \sim q(\mathbf{x}_t \mid \mathbf{x}_0)} 
  \left[ \lVert \mathbf{x}_0 - \bm{D}_{\bm{\theta}}(\mathbf{x}_t, \sigma_t) \rVert^2\right],
 \label{loss function of HTDMs 1}
\end{align}
where $\mathcal{U}(\{\Delta t, 2 \Delta t, \ldots,T\})$ denotes the uniform distribution over discrete time steps.

\subsection{Sampling of HTDMs}
\label{subsec:htdm_sampling}

Given a trained denoiser, sampling starts from $\mathbf{x}_T \sim p(\mathbf{x}_T)$, and a synthetic sample $\mathbf{x}_0$ is obtained by iteratively applying the reverse kernel $p_{\bm{\theta}}(\mathbf{x}_{t-\Delta t} \mid \mathbf{x}_t)$ defined in \eqref{HTDMs model1}.
From this discrete-time reverse kernel, Pandey et al.\ \cite{pandey2024heavytailed} derive both an ODE and an SDE for HTDMs in the continuous-time limit $\Delta t \to 0$.
That is, for practical sampling, the generative process is defined and implemented as a Markov process, even though the diffusion process is non-Markovian.
In our setting, their ODE is derived from a particular choice of $c_t^2$ in \eqref{barsigma 1}. 
With $c_t^2 = \sigma_t \sigma_{t-\Delta t}$, the variance of the reverse kernel, $\bar{\sigma}_t^2 = \sigma_{t-\Delta t}^2 - c_t^4 / \sigma_t^2$, vanishes.
Consequently, the reverse kernel becomes deterministic, and the continuous-time limit yields the following ODE:
\begin{align}
  \dv{\mathbf{x}_t}{t}
  = \frac{\dot{\mu}_t}{\mu_t} \mathbf{x}_t + \left[\frac{\dot{\sigma}_t}{\sigma_t} - \frac{\dot{\mu}_t}{\mu_t}\right] \left(\mathbf{x}_t - \mu_t \bm{D}_{\bm{\theta}}(\mathbf{x}_t, \sigma_t) \right).
  \label{HTDMs ODE1}
\end{align}
They also derive the SDE from the reverse kernel in \eqref{HTDMs model1}, but do not further explore it as a practical sampler.
In the next section, we focus on this SDE perspective and develop a practical SDE-based sampler that explicitly incorporates the state-dependent factor.

\section{Self-Regulating Annealing in HTDMs}
\label{sec:self_regulating_sde}

\subsection{SDE-Based Sampler for HTDMs}

In this subsection, we derive an SDE-based sampler by taking the continuous-time limit $\Delta t \to 0$ of the reverse kernel in \eqref{HTDMs model1}.
In this limit, the generative process is described by a Wiener-driven SDE. 
We argue that reproducing heavy-tailed behavior in this SDE requires a state-dependent factor.
To retain a non-vanishing noise term in the continuous-time limit, we choose $c_t^2$ in \eqref{barsigma 1} as $c_t^2 = \sigma_{t-\Delta t}^2$.
Under a smoothness assumption on $\sigma_t$, this choice yields
\begin{align}
  \bar{\sigma}_t^2 = \sigma_{t-\Delta t}^2 - \frac{c_t^4}{\sigma_t^2} = 2 \sigma_t \dot{\sigma}_t \Delta t + \order{\Delta t^2}.
\end{align}
Hence, the variance of the reverse kernel is of order $\Delta t$, so that the diffusion term remains in the continuous-time limit.
Although the state-dependent factor $(\nu + \Delta_t^2) / (\nu + d)$ appearing in the posterior covariance in \eqref{HTDMs posterior1} is omitted during training, we incorporate it during sampling in our SDE-based sampler.
Its explicit incorporation distinguishes our approach from the previous study \cite{pandey2024heavytailed}.
Using $c_t^2=\sigma_{t-\Delta t}^2$ and a reparameterization of the Student's $t$-distribution,
we can rewrite the reverse transitions as
\begin{small}
  \begin{align}
    \mathbf{x}_{t - \Delta t} =& \mathbf{x}_t + \left[\frac{\dot{\mu}_t}{\mu_t} \mathbf{x}_t - \left(\frac{\dot{\mu}_t}{\mu_t} - \frac{2 \dot{\sigma}_t}{\sigma_t}\right)(\mathbf{x}_t - \mu_t \bm{D}_{\bm{\theta}}(\mathbf{x}_t, \sigma_t))\right](-\Delta t) \notag\\
    &+ \sqrt{\frac{\nu + \Delta_t^2}{\nu + d}} \sqrt{2 \sigma_t \dot{\sigma}_t} \sqrt{\Delta t}\frac{\mathbf{z}_t}{\sqrt{\kappa_t}}, \label{discrete generative process1}
  \end{align}
\end{small}where
\begin{align}
  \Delta_t^2
  &\coloneq \frac{1}{\sigma_t^2}
    \left\lVert \mathbf{x}_t - \mu_t \mathbb{E}[\mathbf{x}_0 \mid \mathbf{x}_t]\right\rVert^2
    \approx \frac{1}{\sigma_t^2}
    \left\lVert \mathbf{x}_t - \mu_t \bm{D}_{\bm{\theta}}(\mathbf{x}_t, \sigma_t) \right\rVert^2.
  \label{state-dependent term2}
\end{align}
Here, $\bm{D}_{\bm{\theta}}(\mathbf{x}_t, \sigma_t)$ is the denoiser trained to minimize the loss function in \eqref{loss function of HTDMs 1}.
Let $\mathbf{z}_t \sim \mathcal{N}_d(\mathbf{0}, \mathbf{I}_d)$ and $\kappa_t \sim \chi^2(\nu+d) / (\nu+d)$ be independent, where $\mathcal{N}_d(\mathbf{0}, \mathbf{I}_d)$ denotes the $d$-dimensional standard Gaussian distribution and $\chi^2(\nu+d)$ denotes the chi-squared distribution with $\nu+d$ degrees of freedom.
Then, $\mathbf{z}_t/\sqrt{\kappa_t} \sim \mathrm{St}_d(\bm{0},\mathbf{I}_d,\nu+d)$.

We next show that in the continuous-time limit $\Delta t \to 0$, the discrete-time update in \eqref{discrete generative process1} converges to an It\^{o} SDE driven by a Wiener process.
To this end, we consider a general forward-time update with Student's $t$ noise:
\begin{align}
  \Delta \mathbf{x}_t
  &\coloneq \mathbf{x}_{t+\Delta t} - \mathbf{x}_t
   = \bm{f}(\mathbf{x}_t, t) \Delta t
     + g(\mathbf{x}_t, t)\sqrt{\Delta t} \frac{\mathbf{z}_t}{\sqrt{\kappa_t}}, \\
  \mathbf{z}_t &\sim \mathcal{N}_d(\mathbf{0}, \mathbf{I}_d), \quad
  \kappa_t = \frac{\xi_t}{\nu + d},\quad
  \xi_t \sim \chi^2(\nu + d),
\end{align}
where $\bm{f}(\mathbf{x}_t,t) \in \mathbb{R}^d$ is a $d$-dimensional vector-valued function and $g(\mathbf{x}_t,t)\in\mathbb{R}$ is a scalar-valued function.
The random variables $\mathbf{z}_t$ and $\kappa_t$ are independent, and the pairs $(\mathbf{z}_t,\kappa_t)$ are i.i.d.\ over $t$.

\begin{proposition}[Gaussian limit of Student's $t$ increments]
\label{prop:gaussian_limit}
Assume $\nu > 2$.
For fixed $\mathbf{x}_t=\bm{x}$ and $t$, the conditional distribution of $\Delta \mathbf{x}_t$ is approximated by a Gaussian distribution to first order in $\Delta t$.
\begin{align}
  p(\Delta \mathbf{x}_t \mid \mathbf{x}_t = \bm{x})
  \approx \mathcal{N}_d \left(
    \bm{f}(\bm{x}, t)\Delta t, g^2(\bm{x}, t) \Delta t \frac{\nu + d}{\nu + d - 2} \mathbf{I}_d
  \right).
  \label{eq:gaussian_increment_approx}
\end{align}
Consequently, in the continuous-time limit $\Delta t \to 0$, 
the resulting dynamics are described by an It\^{o} SDE driven by a $d$-dimensional Wiener process:
\begin{align}
  \dd \mathbf{x}_t = \bm{f}(\mathbf{x}_t,t)\dd t + g(\mathbf{x}_t,t)\sqrt{\frac{\nu + d}{\nu + d - 2}} \dd \mathbf{w}_t.
\end{align}
\end{proposition}

\begin{proof}
  Fix $(\bm{x},t)$ and set $\bm{f} \coloneq \bm{f}(\bm{x},t)$, $g \coloneq g(\bm{x},t)$.
  Define the characteristic function of $\Delta \mathbf{x}_t$ conditioned on $\mathbf{x}_t=\bm{x}$ as
  \begin{align}
    \phi(\bm{u}\mid \bm{x})
    \coloneq
    \mathbb{E}_{\mathbf{z}_t, \kappa_t}\left[\e^{\ii \bm{u}^\top \Delta\mathbf{x}_t}\mid \mathbf{x}_t=\bm{x}\right].
  \end{align}
  Using $\Delta \mathbf{x}_t = \bm{f}\Delta t + g \sqrt{\Delta t} \frac{\mathbf{z}_t}{\sqrt{\kappa_t}}$ 
  and $\mathbf{z}_t\sim\mathcal{N}_d(\bm{0},\mathbf{I}_d)$, we first take the conditional expectation over $\mathbf{z}_t$ given $\kappa_t$:
  \begin{align}
    \phi(\bm{u}\mid \bm{x})
    &= \e^{\ii \bm{u}^\top \bm{f}\Delta t}
       \mathbb{E}_{\kappa_t}\left[
        \exp\left(-\frac{g^2 \Delta t}{2\kappa_t}\|\bm{u}\|^2\right)\right] \notag\\
    &= 1 + \ii \bm{u}^\top \bm{f}\Delta t
       -\frac{1}{2}g^2 \Delta t \mathbb{E} \left[\kappa_t^{-1}\right] \|\bm{u}\|^2 + o(\Delta t).
  \end{align}
  Since $\kappa_t =\xi_t /(\nu+d)$ and $\xi_t \sim \chi^2(\nu+d)$,
  \begin{align}
    \mathbb{E} \left[\kappa_t^{-1}\right] = (\nu+d)\mathbb{E}\left[\xi_t^{-1}\right] =\frac{\nu+d}{\nu+d-2}.
  \end{align}
  We therefore have
  \begin{small}
    \begin{align}
      &\phi(\bm{u}\mid \bm{x}) \notag\\
      =& 1+ \ii \bm{u}^\top\bm{f}\Delta t -\frac{1}{2}\bm{u}^\top \left(g^2 \Delta t \frac{\nu+d}{\nu+d-2} \mathbf{I}_d \right) \bm{u} + o(\Delta t),
    \end{align}
  \end{small}which is the characteristic function of the Gaussian distribution in \eqref{eq:gaussian_increment_approx} up to $o(\Delta t)$.
\end{proof}

Applying Proposition \ref{prop:gaussian_limit} to the reverse-time update in \eqref{discrete generative process1}, and noting that sampling proceeds backward in time from $T$ to $0$, we obtain the following reverse-time SDE:
\begin{small}
  \begin{align}
    &\dd \mathbf{x}_t = \left[\frac{\dot{\mu}_t}{\mu_t} \mathbf{x}_t - \left(\frac{\dot{\mu}_t}{\mu_t} - \frac{2 \dot{\sigma}_t}{\sigma_t}\right)(\mathbf{x}_t - \mu_t \bm{D}_{\bm{\theta}}(\mathbf{x}_t, \sigma_t))\right] \dd t \notag \\
    &\qquad \quad+ \alpha(\Delta_t^2(\mathbf{x}_t), t) \sqrt{2 \sigma_t \dot{\sigma}_t} \dd \mathbf{w}_t, \\
    &\alpha(\Delta_t^2(\mathbf{x}_t), t) \coloneq \sqrt{\frac{\nu + \Delta_t^2}{\nu + d - 2}} \label{state-dependent factor1}.
  \end{align}
\end{small}We refer to $\alpha(\Delta_t^2(\mathbf{x}_t), t)$ as the state-dependent coefficient.
We can generate samples by numerically solving this SDE from $t=T$ to $t=0$, starting from $\mathbf{x}_T \sim p(\mathbf{x}_T) = \mathrm{St}_d(\bm{0}, \sigma_T^2 \mathbf{I}_d, \nu)$.

For example, choosing $\mu_t = 1$ and $\sigma_t = \sigma\sqrt{t}$ yields
\begin{align}
  \dd \mathbf{x}_t &= \frac{\mathbf{x}_t - \bm{D}_{\bm{\theta}}(\mathbf{x}_t, \sigma_t)}{t} \dd t + \sigma \alpha(\Delta_t^2(\mathbf{x}_t), t) \dd \mathbf{w}_t, \label{VE sde sampler 1}\\
  \Delta_t^2 &= \frac{1}{\sigma^2 t} \lVert \mathbf{x}_t - \bm{D}_{\bm{\theta}}(\mathbf{x}_t, \sigma_t) \rVert^2.
\end{align}
This parameter choice corresponds to the variance-exploding (VE) setting in HTDMs.
In the Gaussian limit $\nu \to \infty$ of the Student's $t$-distribution, the state-dependent coefficient converges to $1$, and the SDE reduces to the reverse-time VE-SDE in \cite{song2020score} with constant diffusion coefficient $\sigma$.
Algorithm~\ref{alg:sde_sampler_ve} summarizes a discretized VE-SDE-based sampler for HTDMs.
In implementation, we index the states by the noise level $\sigma$ rather than time: 
we denote by $\mathbf{x}_i$ the state at noise level $\sigma_i$ (with $\sigma_N=\sigma_{\text{max}}$ and $\sigma_1=\sigma_{\text{min}}$), and set $\mathbf{x}_0$ to be the final denoised output.
\begin{algorithm}
  \caption{SDE-based sampler ($\mu_t=1, \sigma_t=\sigma \sqrt{t}$)}
  \label{alg:sde_sampler_ve}
  \begin{algorithmic}[1]
  \Require Denoiser $\bm{D}_{\bm{\theta}}(\mathbf{x},\sigma)$, degrees of freedom $\nu$, dimension $d$, number of steps $N$,
           noise-level sequence $\{\sigma_i\}_{i=1}^{N}$ ($\sigma_{N}=\sigma_{\text{max}}>\cdots>\sigma_{1}=\sigma_{\text{min}}$)
  \State Sample $\mathbf{x}_N \sim \mathrm{St}_d\!\left(\mathbf{0}, \sigma_N^2 \mathbf{I}_d, \nu\right)$ \Comment{Initialization}
  \For{$i=N$ \textbf{down to} $2$}
      \State $\bm{d}_i \gets \dfrac{2(\mathbf{x}_{i} - \bm{D}_{\bm{\theta}}(\mathbf{x}_i,\sigma_i))}{\sigma_i}$ \Comment{Drift term}
      \State $\Delta_i^2 \gets \dfrac{\|\mathbf{x}_i - \bm{D}_{\bm{\theta}}(\mathbf{x}_i,\sigma_i)\|^2}{\sigma_i^2}$
      \State $\alpha_i \gets \sqrt{\dfrac{\nu+\Delta_i^2}{\nu + d - 2}}$ \Comment{State-dependent coefficient}
      \State Sample $\mathbf{z}_i \sim \mathcal{N}_d\!\left(\mathbf{0}, \mathbf{I}_d\right)$
      \State $\mathbf{x}_{i-1} \gets \mathbf{x}_{i}
        + \bm{d}_i\,(\sigma_{i-1}-\sigma_i)
        + \alpha_i\,\sqrt{\sigma_i^2-\sigma_{i-1}^2}\,\mathbf{z}_i$
  \EndFor
  \State $\mathbf{x}_0 \gets \bm{D}_{\bm{\theta}}(\mathbf{x}_{1}, \sigma_{1})$ \Comment{Final denoising}
  \State \Return $\mathbf{x}_0$
  \end{algorithmic}
\end{algorithm}

\subsection{Self-Regulating Annealing in the SDE-Based Sampler}
In this subsection, we analyze our proposed VE-SDE-based sampler in \eqref{VE sde sampler 1} and elucidate its self-regulating annealing mechanism.
We assume that the denoiser $\bm{D}_{\bm{\theta}}(\mathbf{x}_t, \sigma_t)$ has been trained to accurately approximate $\mathbb{E}[\mathbf{x}_0 \mid \mathbf{x}_t]$.
In other words, given a noisy state $\mathbf{x}_t$, the denoiser estimates a corresponding target (denoised) state.
If we ignore the stochastic term in \eqref{VE sde sampler 1} and consider only the drift, the dynamics are driven toward this target estimate.
In particular, the current state is consistent with the target estimate when it satisfies the following fixed-point equation:
\begin{align}
  \mathbf{x}_t - \bm{D}_{\bm{\theta}}(\mathbf{x}_t,\sigma_t) = 0.
\end{align}

On the other hand, the diffusion coefficient contains
$\Delta_t^2 \propto \|\mathbf{x}_t-\bm{D}_{\bm{\theta}}(\mathbf{x}_t,\sigma_t)\|^2$,
which measures the distance between the current state and the target estimate at time $t$.
Hence, when the current state is far from the target estimate, the diffusion coefficient increases, injecting stronger noise and promoting broader exploration; 
when it is close, the diffusion coefficient decreases, leading to milder noise and more gradual exploration.
Equivalently, this mechanism can be interpreted as a state-dependent modulation of the effective annealing temperature, which we refer to as a \emph{self-regulating annealing mechanism} \cite{aguilera2025explosive}.

We next demonstrate the self-regulating annealing mechanism in a setting where the SDE dynamics can be analyzed explicitly.
Consider the symmetric two-point data distribution 
$p_{\mathrm{data}}(\mathbf{x}_0) = \tfrac{1}{2} \delta(\mathbf{x}_0 - \bm{a}) + \tfrac{1}{2} \delta(\mathbf{x}_0 + \bm{a})$,
for some fixed $\bm{a} \in \mathbb{R}^d$.
In this case, $\mathbb{E}[\mathbf{x}_0 \mid \mathbf{x}_t]$ can be obtained in closed form as
\begin{align}
  \mathbb{E}[\mathbf{x}_0 \mid \mathbf{x}_t] &= \bm{a} \tanh_{\gamma}\left(\beta_t(\mathbf{x}_t) \bm{a}^\top \mathbf{x}_t\right), \\
  \beta_t(\mathbf{x}_t) &\coloneq \frac{\nu + d}{\nu \sigma^2 t} \cdot \frac{1}{1 + \frac{1}{\nu \sigma^2 t} (\lVert\mathbf{x}_t \rVert^2 + \lVert\bm{a} \rVert^2)},
\end{align}
where $\gamma \coloneq -\frac{2}{\nu+d}$.
Here, $\tanh_{\gamma}(x)$ denotes the deformed hyperbolic tangent defined as
\begin{align}
  \tanh_{\gamma}(x) \coloneq \frac{(1 + \gamma x)^{\frac{1}{\gamma}} - (1 - \gamma x)^{\frac{1}{\gamma}}}{(1 + \gamma x)^{\frac{1}{\gamma}} + (1 - \gamma x)^{\frac{1}{\gamma}}} \quad \text{for $|\gamma x| < 1$},
\end{align}
which converges to $\tanh(x)$ as $\gamma \to 0$ (equivalently, as $\nu \to \infty$).
Assuming an ideal denoiser $\bm{D}_{\bm{\theta}}(\mathbf{x}_t,\sigma_t) = \mathbb{E}[\mathbf{x}_0 \mid \mathbf{x}_t]$ under this data distribution,
the SDE in \eqref{VE sde sampler 1} can be written as
\begin{align}
  \dd \mathbf{x}_t &= \frac{\mathbf{x}_t - \bm{a} \tanh_{\gamma}\left(\beta_t(\mathbf{x}_t) \bm{a}^\top \mathbf{x}_t\right)}{t} \dd t + \sigma \alpha(\Delta_t^2(\mathbf{x}_t), t) \dd \mathbf{w}_t, \label{VE sde sampler 2}\\
  \Delta_t^2 &= \frac{1}{\sigma^2 t} \lVert \mathbf{x}_t - \bm{a} \tanh_{\gamma}\left(\beta_t(\mathbf{x}_t) \bm{a}^\top \mathbf{x}_t\right) \rVert^2.
\end{align}

Figure~\ref{fig:ve_sde_overview} summarizes the key behavior of \eqref{VE sde sampler 2} in the one-dimensional setting with $a=10$, $\nu=3$, and $\sigma=1$.
Figure~\ref{fig:ve_fp} plots $y=x$ and $y=a \tanh_{\gamma}\bigl(\beta_t(x) a x\bigr)$ for a representative value of $t \in [0,1]$.
Their intersections correspond to the fixed points (zeros of the drift), which satisfy the following fixed-point equation:
\begin{align}
  x = a \tanh_{\gamma} \bigl(\beta_t(x) a x\bigr). \label{fixed point equation2}
\end{align}
This equation is analogous to the self-consistent equations in statistical physics, and can be interpreted as the HTDM counterpart of the result in \cite{ambrogioni2023statistical, raya2024spontaneous}.
In the one-dimensional case, $\Delta_t^2$ in \eqref{VE sde sampler 2} corresponds to the squared vertical gap between the two curves in Figure~\ref{fig:ve_fp}, scaled by the noise level.
Figure~\ref{fig:ve_traj} shows sample trajectories of \eqref{VE sde sampler 2} over $t \in [0, 1]$ (solid) and those of an ablated baseline with $\alpha(\Delta_t^2(\mathbf{x}_t), t) \equiv 1$ (dashed), using the same Wiener increments for a fair comparison. 
Since $\alpha(\Delta_t^2(\mathbf{x}_t), t)$ increases with $\Delta_t^2$, the trajectories exhibit broader exploration when the state is far from the fixed points.

\begin{figure}[t]
  \centering

  \subfloat[]{
    \includegraphics[height=0.17\textheight]{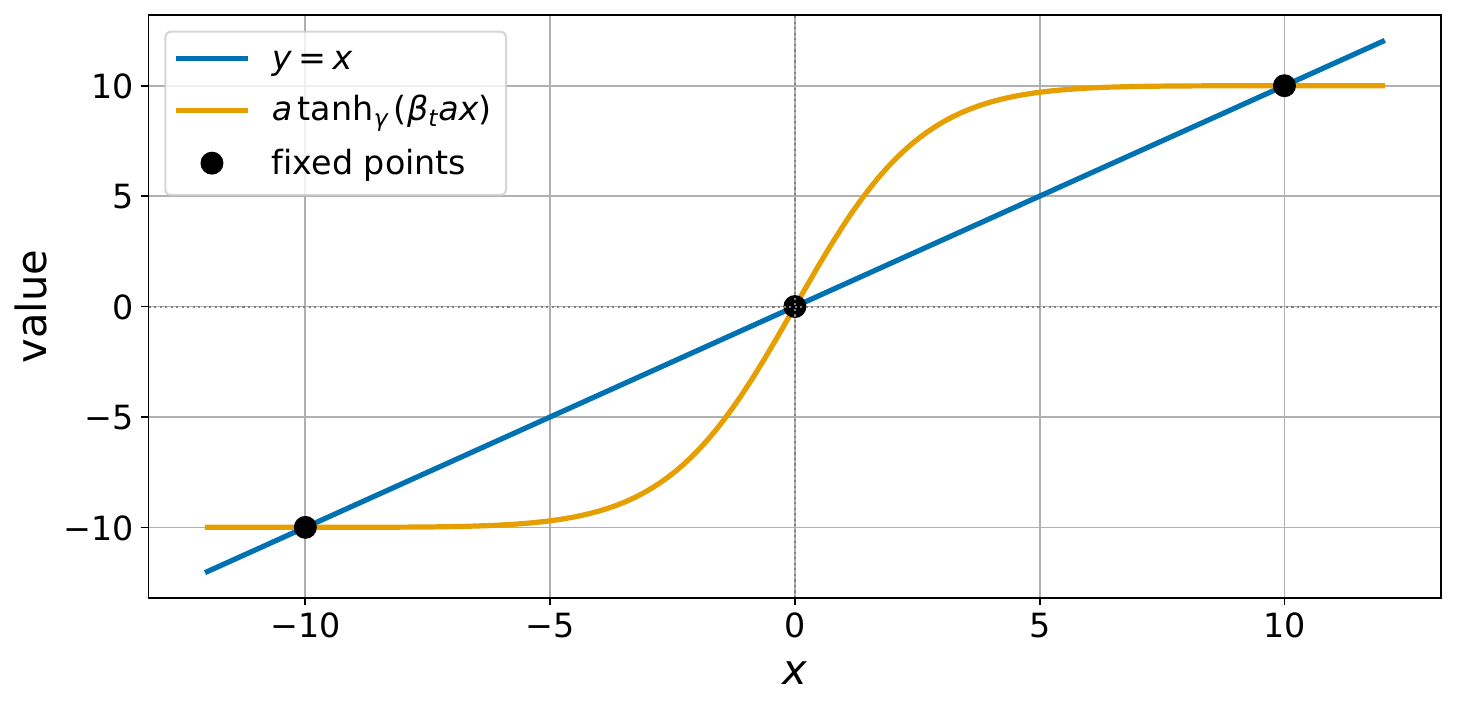}
    \label{fig:ve_fp}
  }

  \vspace{0.3em}

  \subfloat[]{
    \includegraphics[height=0.18\textheight]{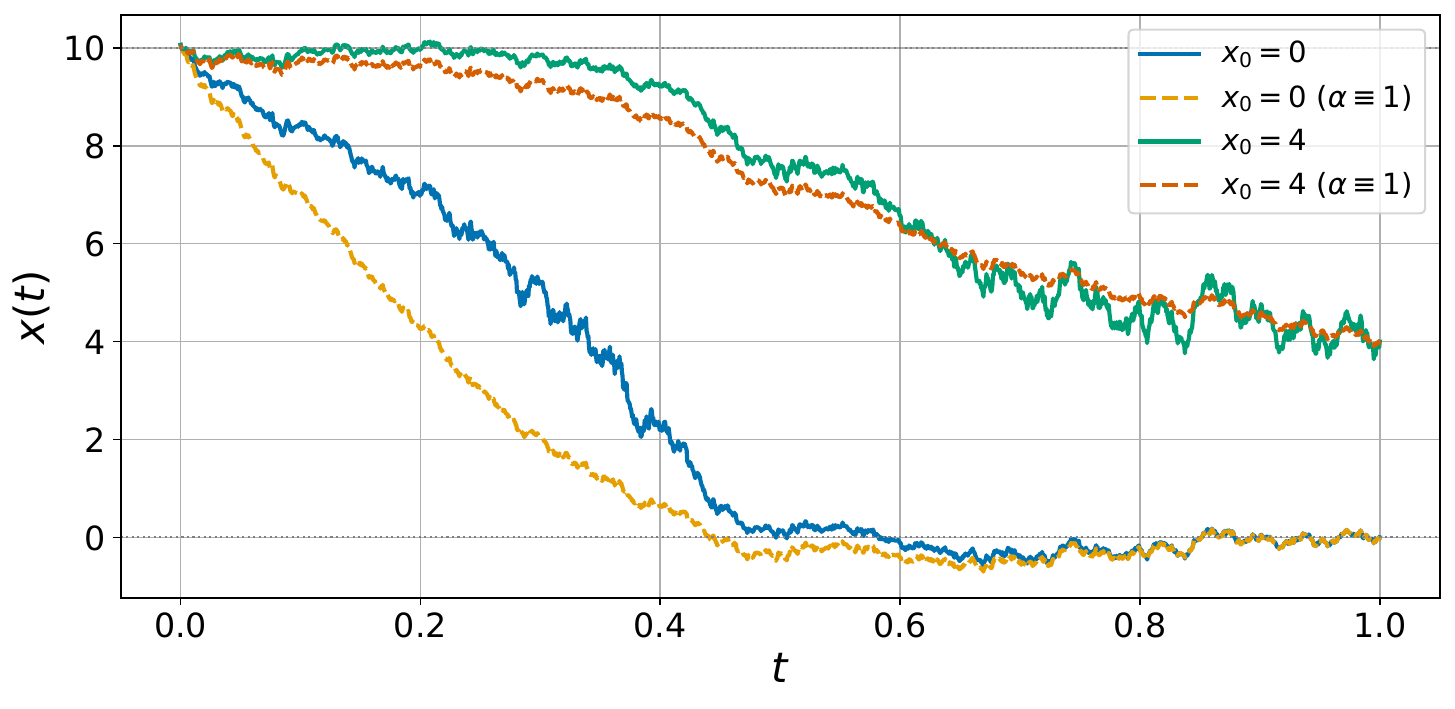}
    \label{fig:ve_traj}
  }

  \caption{Self-regulating annealing mechanism in \eqref{VE sde sampler 2} for the symmetric two-point data distribution in the one-dimensional setting ($a=10, \nu=3, \sigma=1$).
  (a) Plots of $y=x$ and $y=a\tanh_{\gamma}(\beta_t(x)ax)$ for a representative value of $t \in [0,1]$; their intersections correspond to the fixed points (zeros of the drift) satisfying \eqref{fixed point equation2}. 
  The curves change only slightly over $t\in[0,1]$; hence, the curve at $t=1$ is shown.
  (b) Sample trajectories over $t\in[0,1]$: \eqref{VE sde sampler 2} (solid) versus an ablated baseline with $\alpha(\Delta_t^2(\mathbf{x}_t),t)\equiv 1$ (dashed), using the same Wiener increments for a fair comparison.
  }
  \label{fig:ve_sde_overview}
\end{figure}

\section{EXPERIMENTS}
We evaluate generative performance on a synthetic heavy-tailed dataset.
To directly compare with the theoretical tail behavior, we use a one-dimensional Student's $t$-distribution as the data distribution:
\begin{align}
p_{\mathrm{data}}(x) = \mathrm{St}_1(x \mid 0, 1, 3).
\label{eq:synthetic_target}
\end{align}
We evaluate the performance of the following four samplers:
\begin{enumerate}
  \item \textbf{VE-SDE (baseline)}: a Gaussian VE-SDE sampler for EDM\cite{song2020score, karras2022elucidating}.
  \item \textbf{$t$-ODE}: an ODE-based sampler for $t$-EDM \cite{pandey2024heavytailed}.
  \item \textbf{$t$-SDE}: our proposed SDE-based sampler for $t$-EDM that explicitly incorporates the state-dependent coefficient $\alpha(\Delta_t^2(\mathbf{x}_t), t)$.
  \item \textbf{Ablated $t$-SDE}: an ablated variant of our SDE-based sampler for $t$-EDM in which $\alpha(\Delta_t^2(\mathbf{x}_t), t) \equiv 1$.
\end{enumerate}
We focus on tail fidelity, using quantile--quantile (Q--Q) plots for qualitative evaluation and extreme-tail probabilities for quantitative comparison with the ground truth.

\subsection{Experimental Setup}
\subsubsection{Dataset}
We generate $10^6$ i.i.d.\ samples from $p_{\mathrm{data}}$ for training and another $10^6$ i.i.d.\ samples for testing, and apply $z$-score normalization.

\subsubsection{Denoiser Architecture}
We use the same denoiser architecture for all methods: a fully connected neural network with two hidden layers of width 128.

\subsubsection{Preconditioning and Training}
We follow the preconditioning and training settings in Table~7 of \cite{pandey2024heavytailed}: EDM settings for VE-SDE and $t$-EDM settings for $t$-ODE and both $t$-SDE variants.
We train both EDM and $t$-EDM on a fixed budget of $10^7$ samples and use the final checkpoint.

\subsubsection{Sampling}
VE-SDE and the two $t$-SDE variants are all based on the SDE in \eqref{VE sde sampler 1}.
For VE-SDE, we use the Gaussian limit $(\nu \to \infty)$.
The $t$-SDE retains the state-dependent coefficient $\alpha(\Delta_t^2(\mathbf{x}_t), t)$, whereas the ablated $t$-SDE sets it to $1$.
The $t$-ODE is based on the ODE in \eqref{HTDMs ODE1} under the VE parameterization, i.e., $\mu_t = 1$ and $\sigma_t = \sigma\sqrt{t}$.
We choose the discrete noise levels $\{\sigma_i\}_{i=1}^N$ according to the EDM schedule of \cite{karras2022elucidating}.
We solve the ODE using Heun's method and the SDE using the Euler--Maruyama method.

\subsubsection{Hyperparameters}
Unless otherwise noted, we follow the hyperparameter settings in Table~7 of \cite{pandey2024heavytailed}.
When training $t$-EDM, we set $\nu=3$.
For our SDE-based sampler, $\nu$ appears explicitly in both the prior and state-dependent coefficient. 
We therefore treat it as a sampling hyperparameter and sweep $\nu \in \{2.5, 3, 3.5, 4\}$, ultimately using $\nu=2.5$ based on the evaluation criteria described below.
Since the number of function evaluations (NFE) differs between Heun's method and the Euler--Maruyama method, we choose the ODE and SDE step counts to match their NFE budgets, using 64 ODE steps and 128 SDE steps.

\subsection{Evaluation}
We first evaluate overall sample quality by computing the 1-Wasserstein distance $W_1$ from $10^6$ generated and $10^6$ test samples.
We next assess tail fidelity using (i) Q--Q plots, which visualize the alignment of empirical quantiles between generated and test samples, 
and (ii) extreme-tail probabilities in the raw data space.
For Q--Q plots, we use $10^{6}$ generated samples and $10^{6}$ test samples.
We compute empirical quantiles on a uniform grid of $10^{5}$ quantile levels in $[10^{-4}, 1-10^{-4}]$ to prevent extreme quantiles from distorting the axis range.
This trimming is only for visualization; we still evaluate the extreme tails quantitatively using extreme-tail probabilities.

For extreme-tail probabilities, we measure the empirical rate of samples exceeding a fixed two-sided threshold $u$.
This comparison is performed in the raw data space.
We set $u$ to the theoretical two-sided $0.999$ quantile of the ground-truth distribution, i.e., $u \coloneq u_{0.999}=12.924$, so that $\mathbb{P}_{x\sim p_{\mathrm{data}}}(|x|>u)=10^{-3}$.
For each method, we generate $10^7$ samples and split them into 10 disjoint batches of $M = 10^6$ samples.
Within each batch, we estimate $\mathbb{P}(|x|>u)$ by the empirical exceedance rate $\widehat{p}=\frac{1}{M}\sum_{i=1}^{M}\mathbf{1}[|x_i|>u]$, where $\mathbf{1}[\cdot]$ denotes the indicator function.
We report the mean and standard deviation of $\widehat{p}$ across batches, as well as the relative error with respect to the ground truth.

\subsection{Results}
\begin{figure}
  \centering
  \includegraphics[height=0.3\textheight]{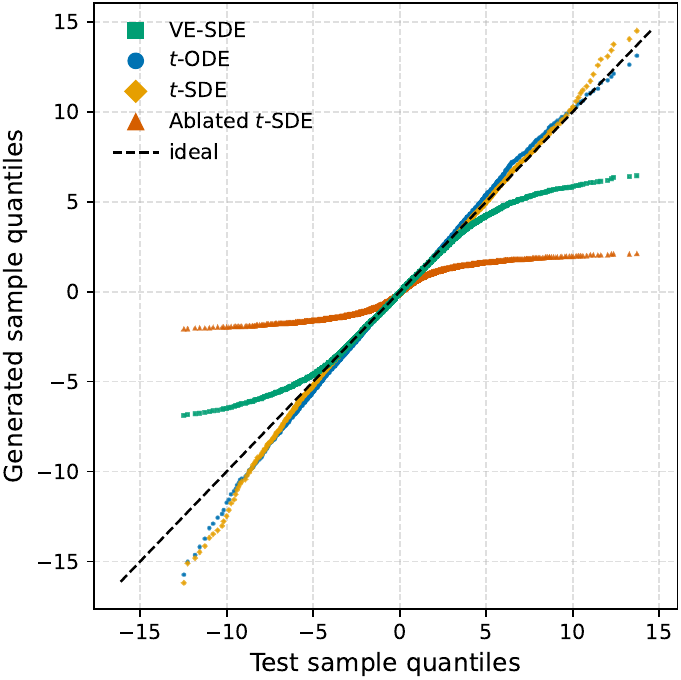}
  \caption{Q--Q plots comparing generated and test samples for the four samplers.}
  \label{fig:qq_plots}
\end{figure}

Figure~\ref{fig:qq_plots} shows the Q--Q plots of the four samplers.
The VE-SDE and the ablated $t$-SDE deviate noticeably from the ideal $y=x$ line in the tail region, indicating poor tail fidelity.
In contrast, $t$-ODE and $t$-SDE closely follow the ideal line, demonstrating superior tail fidelity.

Table~\ref{tab:overall_tail_compact} shows that overall sample quality is broadly comparable across methods, except for the ablated $t$-SDE, which exhibits a markedly larger $\mathcal{W}_1$.
In contrast, tail fidelity differs substantially across methods: the $t$-ODE and $t$-SDE reproduce the correct order of the extreme-tail probability ($10^{-3}$), whereas the VE-SDE severely underestimates this probability and the ablated $t$-SDE reduces it to zero.

\begin{table}[t]
  \centering
  \caption{Overall accuracy ($\mathcal{W}_1$) and heavy-tail fidelity at $u_{0.999}$.
  Extreme-tail probabilities are reported as mean $\pm$ std across batches.
  We define the relative error as $\mathrm{Rel.\ Err.} \coloneq (\widehat{p}-10^{-3})/10^{-3}$.}
  \label{tab:overall_tail_compact}
  \setlength{\tabcolsep}{4pt}
  \small
  \begin{tabular}{lccc}
  \hline
  Method & $\mathcal{W}_1$ & $\mathbb{P}(|X|>u_{0.999})$ & Rel.\ Err. \\
  \hline
  VE-SDE & 0.0287 & $(5.630\pm0.689)\times10^{-5}$ & $-0.9437$ \\
  $t$-ODE & 0.0202 & $(1.307\pm0.029)\times10^{-3}$ & $0.3073$ \\
  $t$-SDE & $\mathbf{0.0162}$ & $(1.233\pm0.030)\times10^{-3}$ & $\mathbf{0.2325}$ \\
  Ablated $t$-SDE & 0.2233 & $0$ & $-1$ \\
  \hline
  \end{tabular}
\end{table}

\section{CONCLUSIONS}
We revisit HTDMs from an SDE perspective and propose an SDE-based sampler. 
This sampler is driven by a Wiener process and explicitly incorporates the state-dependent diffusion coefficient induced by the Student's $t$ posterior.
Owing to the state dependence of the diffusion coefficient, the generative dynamics exhibit a self-regulating annealing mechanism: the effective noise scale increases when the current state is far from the denoiser's target estimate and decreases as the state approaches it.
Experiments on a synthetic dataset confirm that this state dependence is necessary not only for faithfully reproducing the heavy tails but also for improving overall sample quality.
As a result, our SDE-based sampler improves tail fidelity relative to a Gaussian baseline and achieves performance comparable to the existing $t$-ODE-based sampler.

\subsection{Limitations and Future Work}
Our experimental evaluation is currently limited to a synthetic one-dimensional setting, and the degrees-of-freedom parameter $\nu$ used in sampling is selected through a small sweep.
Extending both the analysis and the empirical validation to higher-dimensional benchmarks (e.g., images) and developing principled methods for selecting $\nu$ are promising directions for future work.

\section*{Acknowledgements}
This work is supported by JST SPRING, Grant Number A94251400049 and JSPS KAKENHI, Grant Numbers JP 24K21518, 25K03085.

\bibliographystyle{IEEEtran}
\bibliography{IEEE_ref}

\end{document}